\documentclass[twoside]{article}
\usepackage{graphicx}
\usepackage{amsfonts}
\usepackage{amsmath}
\usepackage{amssymb}
\usepackage{ntheorem}

\theoremseparator{.}
\newtheorem{assumption}{Assumption}
\newtheorem{theorem}{Theorem}
\newtheorem{lemma}{Lemma}
\newtheorem{claim}{Claim}

\theoremheaderfont{\itshape}
\theorembodyfont{\normalfont}
\theoremsymbol{}
\newtheorem{remark}{Remark}

\theoremstyle{nonumberplain}
\theoremheaderfont{\itshape}
\theorembodyfont{\normalfont}
\theoremsymbol{}
\newtheorem{proof}{Proof}

\usepackage{algorithm}
\usepackage{algorithmic}

\usepackage[accepted]{aistats2020}
\usepackage[round]{natbib}

\DeclareMathOperator*{\argmin}{arg\,min}

\usepackage{xcolor}

\begin{document}

%

%

\twocolumn[

\aistatstitle{Learning Entangled Single-Sample Distributions via Iterative Trimming}

\aistatsauthor{ Hui Yuan \And Yingyu Liang}

\aistatsaddress{Department of Statistics and Finance
\\University of Science and Technology of China \And  Department of Computer Sciences
\\University of Wisconsin-Madison } ]

\begin{abstract}
In the setting of entangled single-sample distributions, the goal is to estimate some common parameter shared by a family of distributions, given one \emph{single} sample from each distribution. We study mean estimation and linear regression under general conditions, and analyze a simple and computationally efficient method based on iteratively trimming samples and re-estimating the parameter on the trimmed sample set. We show that the method in logarithmic iterations outputs an estimation whose error only depends on the noise level of the $\lceil \alpha n \rceil$-th noisiest data point where $\alpha$ is a constant and $n$ is the sample size. This means it can tolerate a constant fraction of high-noise points. These are the first such results under our general conditions with computationally efficient estimators. It also justifies the wide application and empirical success of iterative trimming in practice. Our theoretical results are complemented by experiments on synthetic data.
\end{abstract}

\section{INTRODUCTION}

This work considers the novel parameter estimation setting called entangled single-sample distributions. Different from the typical i.i.d.\ setting, here we have $n$ data points that are independent, but each is from a different distribution. These distributions are entangled in the sense that they share some common parameter, and our goal is to estimate the common parameter. For example, in the problem of mean estimation for entangled single-sample distributions, we have $n$ data points from $n$ different distributions with a common mean but different variances (the mean and all the variances are unknown), and our goal is to estimate the mean.

This setting is motivated for both theoretical and practical reasons. From the theoretical perspective, it goes beyond the typical i.i.d.\ setting and raises many interesting open questions, even on basic topics like mean estimation for Gaussians. It can also be viewed as a generalization of the traditional mixture modeling, since the number of distinct mixture components can grow with the number of samples. From the practical perspective, many modern applications have various forms of heterogeneity, for which the i.i.d.\ assumption can lead to bad modeling of their data. The entangled single-sample setting provides potentially better modeling. This is particularly the case for applications where we have no control over the noise levels of the samples. For example, the images taken by self-driving cars can have varying degrees of noise due to changing weather or lighting conditions. Similarly, signals collected from sensors on the Internet of Things can come with interferences from a changing environment.

Though theoretically interesting and practically important, few studies exist in this setting. \cite{chierichetti2014learning} considered the mean estimation for entangled Gaussians and showed the existence of a gap between estimation error rates of the best possible estimator in this setting and the maximum likelihood estimator when the variances are known.
\cite{pensia2019estimating} considered means estimation for symmetric, unimodal distributions including the symmetric multivariate case (i.e., the distributions are radially symmetric) with sharpened bounds, and provided extensive discussion on the performance of their estimators in different configurations of the variances. 
These existing results focus on specific family of distributions or focus on the case where most samples are ``high-noise" points.

On the contrary, we focus on the case with a constant fraction of "high-noise" points, which is more interesting in practice. We study multivariate mean estimation and linear regression under more general conditions and analyze a simple and efficient estimator based on iterative trimming. The iterative trimming idea is simple: the algorithm keeps an iterate and repeatedly refines it; each time it trims a fraction of bad points based on the current iterate and then uses the trimmed sample set to compute the next iterate. It is computationally very efficient and widely used in practice as a heuristic for handling noisy data. It can also be viewed as an alternating-update version of the classic trimmed estimator (e.g.,~\cite{huber2011robust}) which typically takes exponential time:
\[
  \hat\theta  = \argmin_{\theta \in \Theta; S \subseteq [n], |S| = \lceil \alpha n \rceil} \sum_{i \in S} \textrm{Loss}_i(\theta)
\]
where $\Theta$ is the feasible set for the parameter $\theta$ to be estimated, $\lceil \alpha n \rceil$ is the size of the trimmed sample set $S$, and $\textrm{Loss}_i(\theta)$ is the loss of $\theta$ on the $i$-th data point (e.g., $\ell_2$ error for linear regression). 

For mean estimation, only assuming the distributions have a common mean and bounded covariances, we show that the iterative trimming method in logarithmic iterations outputs a solution whose error only depends on the noise level of the $\lceil \alpha n \rceil$-th noisiest point for $\alpha \ge 4/5$. More precisely, the error only depends on the $\lceil \alpha n \rceil$-th largest value among all the norms of the $n$ covariance matrices. This means the method can tolerate a $1/5$ fraction of ``high-noise" points. We also provide a similar result for linear regression, under a regularity condition that the explanatory variables are sufficiently spread out in different directions (satisfied by typical distributions like Gaussians). As far as we know, these are the first such results of iterative trimming under our general conditions in the entangled single-sample distributions setting. These results also theoretically justify the wide application and empirical success of the simple iterative trimming method in practice. Experiments on synthetic data provide positive support for our analysis.

\section{RELATED WORK}

\noindent{\bf Entangled distributions.}
This setting is first studied by~\citet{chierichetti2014learning}, which considered mean estimation for entangled Gaussians and presented a algorithm combining the $k$-median and the $k$-shortest gap algorithms. It also showed the existence of a gap between the error rates of the best possible estimator in this setting and the maximum likelihood estimator when the variances are known. \citet{pensia2019estimating} considered a more general class of distributions (unimodal and symmetric) and provided analysis on both individual estimator ($r$-modal interval, $k$-shortest gap, $k$-median estimators) and hybrid estimator, which combines Median estimator with Shortest Gap or Modal Interval estimator. They also discussed slight relaxation of the symmetry assumption and provided extensions to linear regression. 
Our work considers mean estimation and linear regression under more general conditions and analyzes a simpler estimator. However, our results are not directly comparable to the existing ones above, since those focus on the case where most of the points have high noise or have extra constraints on distributions are assumed. For the constrained distributions, our results are weaker than the existing ones.
See the detailed discussion in the remarks after our theorems. 

This setting is also closely related to robust estimation, which have been extensively studied in the literature of both classic statistics and machine learning theory. 

\noindent{\bf Robust mean estimation.}
There are several classes of data distribution models for robust mean estimators. The most commonly addressed is adversarial contamination model, whose origin can be traced back to the malicious noise model by~\citet{valiant1985learning} and the contamination model by~\citet{huber2011robust}. Under contamination, mean estimation has been investigated in~\cite{diakonikolas2017being,diakonikolas2019robust,cheng2019high}. 
Another related model is the mixture of distributions. There has been steady progress in algorithms for leaning mixtures, in particular, leaning Gaussian mixtures. Starting from~\citet{dasgupta1999learning}, a rich collection of results are provided in many studies, such as~\cite{sanjeev2001learning, achlioptas2005spectral, kannan2005spectral,belkin2010polynomial, belkin2010toward, kalai2010efficiently,moitra2010settling,diakonikolas2018robustly}.

\noindent{\bf Robust regression.} 
Robust Least Squares Regression (RLSR) addresses the problem of learning regression coefficients
in the presence of corruptions in the response vector. A class of robust regression estimator solving RLSR is Least Trimmed Square (LTS) estimator, which is first introduced
by~\citet{rousseeuw1984least} and has high breakdown point. The algorithm solutions of LTS are investigated in~\cite{hossjer1995exact, rousseeuw2006computing, shen2013approximate} for the linear regression setting. 
Recently, for robust linear regression in the adversarial setting (i.e., a small fraction of responses are replaced by adversarial values), there is a line of work providing algorithms with theoretical guarantees following the idea of LTS, e.g.,~\cite{bhatia2015robust, vainsencher2017ignoring, yang2018general} for example. For robust linear regression in the adversary setting where both explanatory and response variables can be replaced by adversarial values, a line of work provided algorithms and guarantees, e.g.,~\cite{diakonikolas2018sever, prasad2018robust, klivans2018efficient, shen2019learning}, while some others like~\cite{chen2013robust, balakrishnan2017computationally, liu2018high} considered the high-dimensional scenario.

\section{MEAN ESTIMATION}

Suppose we have $n$ independent samples $\boldsymbol{x}_{i} \sim F_{i} \in \mathbb{R}^{d}$, $d \in \mathbb{N}^{\star}$, where the mean vector and the covariance matrix of each distribution $F_{i}$ exist. Assume $F_{i}$'s have a common mean $\boldsymbol{\mu}^{\star}$ and denote their covariance matrices as $\Sigma_{i}$. When $d=1$, each $\boldsymbol{x}_{i}$ degenerate to an univariate random variable $x_{i}$, and we also write $\boldsymbol{\mu}^{\star}$ as ${\mu}^{\star}$ and write $\Sigma_{i}$ as $\sigma_i^2$. Our goal is to estimate the common mean $\boldsymbol{\mu}^{\star}$. 

\paragraph{Notations}  For an integer $m$, $[m]$ denotes the set $\left\{1, \cdots, m\right\} .$ 
$|S|$ is the cardinality of a set $S$. For two sets $S_{1}, S_{2}$, $S_{1} \backslash S_{2}$ is the set of elements in $S_{1}$ but not in $S_{2}$. $\lambda_{\min}$ and $\lambda_{\max}$ are the minimum and maximum eigenvalues. Denote the order statistics of $\left\{\lambda_{\max}(\Sigma_i)\right\}_{i=1}^{n}$ as $\left\{\lambda_{(i)}\right\}_{i=1}^{n}$. $c$ or $C$ denote constants whose values can vary from line to line.

\subsection{Iterative Trimmed Mean Algorithm}
\label{ME}
First, recall the general version of the least trimmed loss estimator. Let $f_{\boldsymbol{\mu}}\left(\cdot \right)$ be the loss function, given currently learned parameter $\boldsymbol{\mu}$. In contrast to minimizing the total loss of all samples, the least trimmed loss estimator of $\boldsymbol{\mu}^{\star}$ is given by 
\begin{equation}
\label{GTLE}
\hat{\boldsymbol{\mu}}^{(\mathrm{TL})}=\argmin _{\boldsymbol{\mu}, S: |S| = \lceil \alpha n \rceil} \sum_{i \in S} f_{\boldsymbol{\mu}}\left(\boldsymbol{x}_{i} \right) ,
\end{equation}
where $S \subseteq [n]$ and $\alpha \in (0, 1]$ is the fraction of samples we want to fit. Finding  $\hat{\boldsymbol{\mu}}^{(\mathrm{TL})}$ requires minimizing the trimmed loss over both the set of all subsets $S$ with size $\lceil\alpha n\rceil$ and the set of all available values of the parameter $\boldsymbol{\mu}$.  However, solving the minimization above is hard in general. Therefore, we attempt to minimize the trimmed loss in an iterative way by alternating between minimizing over $S$ and $\boldsymbol{\mu}$. That is, it follows a natural iterative strategy: in the $t$-th iteration, first select a subset $S_t$ of samples with the least loss on the current parameter $\boldsymbol{\mu}_t$, and then update to $\boldsymbol{\mu}_{t+1}$ by minimizing the total loss over $S_t$. 

By taking $f_{\boldsymbol{\mu}}\left(\boldsymbol{x} \right)$ in (\ref{GTLE}) as $\left\| \boldsymbol{x}- \boldsymbol{\mu}\right\|^{2}_{2}$, in each iteration, $\lceil\alpha n\rceil$ samples are selected due to their least distance to the current $\boldsymbol{\mu}$ in $l_2$ norm. Besides, note that for a given sample set $S$, 
$$
\argmin _{\boldsymbol{\mu}} \sum_{i \in S} \left\|\boldsymbol{x}_{i}-\boldsymbol{\mu}\right\|^{2}_{2} = \frac{1}{|S|}\sum_{i \in S} \boldsymbol{x}_{i},
$$
that is, the parameter $\boldsymbol{\mu}$ minimizing the total loss over sample set $S$ is the empirical mean over $S$. This leads to our method described in Algorithm~\ref{ITM}, which we referred to as Iterative Trimmed Mean.

Each update in our algorithm is similar to the trimmed-mean (or truncated-mean) estimator, which is, in univariate case, defined by removing a fraction of the sample consisting of the largest and smallest points and averaging over the rest; see~\cite{tukey1963less, huber2011robust, bickel1965some}. A natural generalization to multivariate case is to remove a fraction of the sample consisting of points which have the largest distances to $\boldsymbol{\mu}^{\star}$. The difference between our algorithm and the generalized trimmed-mean estimator is that ours select points based on the estimated $\boldsymbol{\mu}_t$ while trimmed-mean is based on the ground-truth $\boldsymbol{\mu}^\star$. 

\begin{algorithm}[t] 
	\caption{Iterative Trimmed MEAN (ITM)} 
	\label{ITM} 
	\begin{algorithmic}[1]
		\REQUIRE{Samples $\left\{\boldsymbol{x}_{i}\right\}_{i=1}^{n},$ number of rounds $T,$ fraction of samples $\alpha$}
		
		\STATE{$\boldsymbol{\mu}_{0} \leftarrow \frac{1}{n} \sum_{i=1}^n \boldsymbol{x}_{i}$}
		\FOR{$t=0, \cdots, T-1$}
		\STATE Choose samples with smallest current loss: 
		$$S_{t} \leftarrow \argmin _{S :|S|=\lceil\alpha n\rceil} \sum_{i \in S}\left\|\boldsymbol{x}_{i}-\boldsymbol{\mu}_t\right\|^2_{2}$$
		\STATE $\boldsymbol{\mu}_{t+1}=\frac{1}{|S_t|}\sum_{i \in S_t} \boldsymbol{x}_{i}$
		\ENDFOR
		\ENSURE{ $\boldsymbol{\mu}_T$}
	\end{algorithmic}
\end{algorithm}

\subsection{Theoretical Guarantees}


Firstly, we introduce a lemma giving an upper bound of the sum of $\ell_{2}$ distances between points $\boldsymbol{x}_i$ and mean vector $\boldsymbol{\mu}^{\star}$ over a sample set $S$. 

\begin{lemma} 
\label{lemma1}
Define $\lambda_S=\max_{i \in S} \left\{\lambda_{\max}(\Sigma_i)\right\}$. Then we have
\begin{equation}
\label{l1-bound}
\sum_{i \in S}\left\|\boldsymbol{x}_i-\boldsymbol{\mu}^{\star}\right\|_2 \leq 2|S|\sqrt{\lambda_S d}
\end{equation}
with probability at least $1-\frac{1}{|S|}$.
\end{lemma}

\begin{proof}
For each $i \in S$, denote the transformed random vector $\Sigma^{-\frac{1}{2}}_i\left(\boldsymbol{x}_i-\boldsymbol{\mu}^{\star}\right)$ as $\tilde{\boldsymbol{x}}_i$. Then $\tilde{\boldsymbol{x}}_i$ has mean $\boldsymbol{0}$ and identical covariance matrix. Since $\left\|\boldsymbol{x}_i-\boldsymbol{\mu}^{\star}\right\|_2 = \sqrt{\tilde{\boldsymbol{x}}_i^{T}\Sigma_i\tilde{\boldsymbol{x}}_i} \leq \sqrt{\lambda_S} \left\|\tilde{\boldsymbol{x}}_i\right\|_2$, we can write 
$$
\sum_{i \in S}\left\|\boldsymbol{x}_i-\boldsymbol{\mu}^{\star}\right\|_2 \leq \sqrt{\lambda_S}\sum_{i \in S}\left\|\tilde{\boldsymbol{x}}_i\right\|_2.
$$
Let $\xi_i=\left\|\tilde{\boldsymbol{x}}_i\right\|_2$, thus $\left\{\xi_i\right\}_{i \in S}$ are independent and for any $i \in S$, it holds that
\begin{align*}
\mathbb{E}\xi_i \leq \sqrt{\mathbb{E}\xi_i^2}=\sqrt{d},
\quad
\textrm{Var}(\xi_i) \leq \mathbb{E}\xi_i^2 =d.
\end{align*}
By Chebyshev's inequality, we have 
$$
\mathbb{P}\left(\sum_{i \in S}\xi_i \geq 2|S|\sqrt{d} \right) \leq \frac{1}{|S|},
$$
which completes the proof. 
$\blacksquare$ 
\end{proof}

We are now ready to prove our key lemma, which shows the progress of each iteration of the algorithm. The key idea is that the selected subset $S_t$ of samples have a large overlap with the subset $S^{\star}$ of the $\lceil \alpha n \rceil$ ``good" points with smallest covariance. Furthermore, $S_t \backslash S^\star$ is not that ``bad" since by the selection criterion, they have less loss on $\boldsymbol{\mu}_t$ than the points in  $S^\star \backslash S_t$. This thus allows to show the progress.

\begin{lemma} 
\label{m-general}
Given ITM with fraction $\alpha \geq \frac{4}{5}$, we have 
\begin{equation}
\label{m-bound}
\left\|\boldsymbol{\mu}_{t+1}-\boldsymbol{\mu}^{\star}\right\|_{2} \leq \frac{1}{2}\left\|\boldsymbol{\mu}_t-\boldsymbol{\mu}^{\star} \right\|_2 +2\sqrt{d\lambda_{(\lceil\alpha n\rceil)}}
\end{equation}
with probability at least $1-\frac{5}{4n}$.
\end{lemma}

\begin{proof}
\label{p2}
Define $S^{\star} = \left\{i:\lambda_{\max}(\Sigma_i) \leq \lambda_{(\lceil\alpha n\rceil)}\right\}$. Without loss of generality, assume $\alpha n$ is an integer. Then by the algorithm,
$$
\boldsymbol{\mu}_{t+1}=\frac{1}{|S_t|}\sum_{i \in S_t} \boldsymbol{x}_{i} =\boldsymbol{\mu}^{\star}+ \frac{1}{\alpha n}\sum_{i \in S_t} \left(\boldsymbol{x}_{i}-\boldsymbol{\mu}^{\star}\right).
$$

Therefore, the $\ell_{2}$ distance between the learned parameter $\boldsymbol{\mu}_{t+1}$ and the ground truth parameter $\boldsymbol{\mu}^{\star}$ can be bound by:
\begin{align*} 
&~~\left\|\boldsymbol{\mu}_{t+1}-\boldsymbol{\mu}^{\star}\right\|_{2}=\frac{1}{\alpha n}\left\| \sum_{i \in S_t}   \left(\boldsymbol{x}_{i}-\boldsymbol{\mu}^{\star}\right)\right\|_{2}\\
&\leq \frac{1}{\alpha n}\left(\sum_{i \in S_t \cap S^{\star}}\left\|\boldsymbol{x}_i - \boldsymbol{\mu}^{\star}\right\|_{2} + \sum_{i \in S_t \backslash S^{\star}}\left\|\boldsymbol{x}_i - \boldsymbol{\mu}^{\star} \right\|_{2} \right)\\
&\leq \frac{|S_t \backslash S^{\star}| }{\alpha n} \cdot \left\|\boldsymbol{\mu}_t - \boldsymbol{\mu}^{\star}\right\|_{2}\\
&~~~+\frac{1}{\alpha n}\left(\sum_{i \in S^{\star} \cap S_t}\left\|\boldsymbol{x}_i - \boldsymbol{\mu}^{\star}\right\|_{2} + \sum_{i \in S^{\star} \backslash S_t}\left\|\boldsymbol{x}_i - \boldsymbol{\mu}_t \right\|_{2} \right)\\
&\leq \underbrace{\frac{2|S_t \backslash S^{\star}|}{\alpha n}}_{\kappa} \cdot \left\|\boldsymbol{\mu}_t - \boldsymbol{\mu}^{\star}\right\|_{2} + \frac{1}{\alpha n} \sum_{i \in S^{\star}}\left\|\boldsymbol{x}_i - \boldsymbol{\mu}^{\star}\right\|_{2},
\end{align*} 
where the second inequality is guaranteed by line 3 in Algorithm \ref{ITM}. Note that $|S^{\star}| = \alpha n = |S_t|$, thus $|S^{\star} \backslash S_t| = |S_t \backslash S^{\star}|$. Due to the choice of $|S_t|$, we have the distance $\left\|\boldsymbol{x}_{i}-\boldsymbol{\mu}_t \right\|_2$ of each sample in $S_t \backslash S^{\star}$ is less than that of samples in $S^{\star} \backslash S_t$. 

By Lemma \ref{lemma1}, we can bound $\frac{1}{\alpha n}\sum_{i \in S^{\star}}\left\|\boldsymbol{x}_i - \boldsymbol{\mu}^{\star}\right\|_{2}$ by $2\sqrt{d\lambda_{(\lceil\alpha n\rceil)}}$ with probability at least $1-\frac{1}{\alpha n}$. 
Meanwhile, by $\left|S^{\star}\right| = \left|S_{t}\right| =  \alpha n$, it guarantees $\left|S_{t} \backslash S^{\star}\right| \leq (1-\alpha)n.$
Thus, when $\alpha \geq \frac{4}{5}$, $\kappa \leq \frac{2(1-\alpha)}{\alpha} \leq \frac{1}{2}.$
Combining the inequalities completes the proof. $\blacksquare$ 
\end{proof}

Based the error bound per-round in Lemma~\ref{m-general}, it is easy to show that $\left\|\boldsymbol\mu_t-\boldsymbol\mu^{\star}\right\|_2$ can be upper bounded by $\Theta(\sqrt{d\lambda_{(\lceil \alpha n \rceil)}})$ after sufficient iterations.
This leads to our final guarantee.

\begin{theorem}
\label{m-final}
Given ITM with $\alpha \geq \frac{4}{5}$ within $T=\Theta\left(\log_{2} \frac{\left\|\boldsymbol\mu_0-\boldsymbol\mu^{\star}\right\|_2}{\sqrt{d\lambda_{(\lceil \alpha n \rceil)}}}\right)$ iterations, it holds that
\begin{equation}
\label{m-error}
\left\|\boldsymbol\mu_T-\boldsymbol\mu^{\star}\right\|_2 \leq c\sqrt{d\lambda_{(\lceil \alpha n \rceil)}}
\end{equation}
with probability at least $1-\frac{5T}{4n}$.
\end{theorem}

\begin{proof}
By Lemma~\ref{m-general}, we derive
\begin{align*}
\left\|\boldsymbol\mu_T-\boldsymbol\mu^{\star}\right\|_2 &\leq \frac{1}{2}\left\|\boldsymbol{\mu}_{T-1}-\boldsymbol{\mu}^{\star} \right\|_2 +2\sqrt{d\lambda_{(\lceil\alpha n\rceil)}}\\
&\leq \frac{1}{2^T} \left\|\boldsymbol{\mu}_0-\boldsymbol{\mu}^{\star} \right\|_2 + 2\sum_{i=0}^{T-1} \frac{1}{2^i}\sqrt{d\lambda_{(\lceil\alpha n\rceil)}}\\
&\leq \frac{1}{2^T} \left\|\boldsymbol{\mu}_0-\boldsymbol{\mu}^{\star} \right\|_2 + 4 \sqrt{d\lambda_{(\lceil\alpha n\rceil)}}\\
&\leq c\sqrt{d\lambda_{(\lceil\alpha n\rceil)}}
\end{align*}
with probability at least $1-\frac{5T}{4n}$ when $T=\Theta\left(\log_{2} \frac{\left\|\boldsymbol\mu_0-\boldsymbol\mu^{\star}\right\|_2}{\sqrt{d\lambda_{(\lceil \alpha n \rceil)}}}\right)$.
$\blacksquare$ 
\end{proof}

The theorem shows that the error bound only depends on the order statistics $\lambda_{(\lceil \alpha n \rceil)}$, irrelevant of the larger covariances for $\alpha \ge 4/5$. Therefore, using the iterative trimming method, the magnitudes of a $1/5$ fraction of highest noises will not affect the quality of the final solution. However, it is still an open question whether the bound can be made $c\sqrt{d\lambda_{(\lceil\alpha n\rceil)}/n}$, i.e., the optimal rate given an oracle that knows the covariances.

The method is also computationally very simple and efficient. 
Since a single iteration of ITM takes time $\tilde{\mathcal{O}}(nd)$,  $\mathcal{O}(\sqrt{d\lambda_{(\lceil \alpha n \rceil)}})$ error can be computed in $\tilde{\mathcal{O}}(nd)$ time, which is nearly linear in the input size.  

\begin{remark}
\label{r2}
With a further assumption that $\boldsymbol{x}_{i}$ are independently sampled from the Gaussians $\mathcal{N}(\boldsymbol{\mu}^{\star},\Sigma_i)$, the same bound in the theorem holds with probability converging to $1$ with exponential rate of $n$.
This is because $(\ref{l1-bound})$ can be guaranteed with a higher probability. The sketch of the proof follows: $\left(\sum_{i \in S}\left\|\boldsymbol{x}_i-\boldsymbol{\mu}^{\star}\right\|_2\right)^2 \leq |S|\lambda_S\sum_{i \in S}\left\|\tilde{\boldsymbol{x}}_i\right\|^2_2$, by Cauchy-Schwarz inequality. Here $\tilde{\boldsymbol{x}}_i \sim \mathcal{N}\left(0,I_d\right)$ due to the Gaussian distribution of $\boldsymbol{x}_i$.
By the Lemma 3 in~\cite{fan2008sure}, $\sum_{i \in S} \left\|\tilde{\boldsymbol{x}}_i\right\|_2^2 \leq cd|S|$ with probability at least $1-e^{-c^{\prime}d|S|}$. Hence, $\mathbb{P}\left(\left(\sum_{i \in S}\left\|\boldsymbol{x}_i-\boldsymbol{\mu}^{\star}\right\|_2\right)^2 \leq cd|S|^2 \lambda_S\right) \geq 1-e^{-c^{\prime}d|S|}$, where $c$ can be any constant greater than $1$ and $c^{\prime}=[c-1-\log(c)]/2$. 
\end{remark}

\begin{remark}
\label{r3}
For the univariate case with $d=1$, it is sufficient to regard $\Sigma_i$ as one dimensional matrix $\sigma_i^2$. Then we obtain that when $\alpha \geq \frac{4}{5}$ and $T=\Theta\left(\log_{2} \frac{\left\|\mu_0-\mu^{\star}\right\|}{\sigma_{(\lceil\alpha n\rceil)}}\right)$ iterations:
$
\left\|\mu_{T}-\mu^{\star}\right\|_{2} \leq c\sigma_{(\lceil\alpha n\rceil)},
$
with probability at least $1-\frac{5T}{4n}$. 
\end{remark}

\begin{remark}
It is worth mentioning that there is a trade off between the accuracy and running time in ITM. In particular, the constant $\alpha$ can be any constant greater than $\frac{2}{3}$, since it suffices to guarantee $\kappa$ in the proof of Lemma~\ref{m-general} is less than $1$. On the other hand, a smaller $\alpha$ will slow down the speed for $\left\|\boldsymbol\mu_t-\boldsymbol\mu^{\star}\right\|_2$ to shrink, although the computational complexity is still in the same order of $\tilde{\mathcal{O}}(nd)$.
\end{remark}

\begin{remark}We now discuss existing results in detail.

In the univariate case, \cite{chierichetti2014learning} achieved $\min_{2 \leq k\leq \log n}\tilde{\mathcal{O}}\left(n^{1/2(1+1/(k-1))}\sigma_k\right)$ error in time $\mathcal{O}(n\log^2 n)$. Among all estimators studied in~\cite{pensia2019estimating}, the superior performance is obtained by the hybrid estimators, which includes version (1): combining $k_1$-median with $k_2$-shorth and version (2): combining $k_1$-median with modal interval estimator. These two versions achieve similar guarantees while version $1$ has lower run time $\mathcal{O}(n\log n)$. Version $1$ of the hybrid estimator outputs $\hat{\mu}_{k_1, k_2}$ such that $\left|\hat{\mu}_{k_1, k_2}-\mu\right| \leq \frac{4\sqrt{n}\log n}{k_2} r_{2k_2}$ with probability $1-2 \exp(-c^{\prime}k_2)-2 \exp(-c \log^2 n)$, where $k_1=\sqrt{n}\log n$ and $k_2 \geq C \log n$. Since here $r_{k}$ is defined as $\inf\left\{ r: \frac{1}{n} \sum_{i=1}^n \mathbb{P}(|x_i-\mu^{\star}| \leq r)\geq \frac{k}{n} \right\}$, the error bound giving above varies with specific $\left\{F_{i}\right\}_{i=1}^n$, while the worst-case error guarantee is $\mathcal{O}(\sqrt{n} \sigma_{(C \log n)})$. When take $k_2 = \frac{\lceil\alpha n\rceil}{2}$, the error can be $\mathcal{O}(\frac{\log n}{\sqrt{n}} \sigma_{(\lceil\alpha n\rceil)})$. However, this result is for symmetric and unimodal distributions $F_i$'s.

In the multivariate case, 
\cite{chierichetti2014learning} studied the special case where $\boldsymbol{x}_i \sim \mathcal{N}(\boldsymbol{\mu},\sigma^2_i I_d )$ and provided an algorithm with the error bound $ \min_{2 \leq k\leq \log n}\tilde{\mathcal{O}}\left(n^{(1+1/(k-1))/d}\sigma_k\right)$ in time $\tilde{\mathcal{O}}(n^2)$. \cite{pensia2019estimating} mainly considered the special case where the overall mixture distribution is radically symmetric, and sharpens the bound above, resulting in the worst-case error bound $\mathcal{O}(\sqrt{d}\sqrt{n}^{1/d}\sigma_{(Cd \log n)})$. \cite{pensia2019estimating} also provided computationally efficient estimator with running time $\mathcal{O}(n^2 d)$.

These existing results depend on $\sigma_{(C \log n)}$ (or alike) at the expense of an additional factor of roughly $\sqrt{n}^{1/d}$, which is most relevant when $C \log n$ of the points have small noises, i.e., when the samples are dominated by high noises. Our results depend on $\sigma_{(\alpha n)}$ (or $\sqrt{\lambda_{(\lceil\alpha n\rceil)}}$ for multivariate) for $\alpha \ge 4/5$, which is more relevant when $1/5$ fraction of the points have high noises. 

For the general setting we consider, we note that the coordinate-wise median (and thus hybrid estimators as in~\cite{pensia2019estimating}) can be used to achieve matching bounds. 
The robust mean estimation methods in the adversarial contamination model can be applied to our setting and get rid of the $\sqrt{d}$, by viewing the sample points with covariance larger than $\lambda_{(\lceil\alpha n\rceil)}$ as adversarial outliers. (The theorems are typically stated for i.i.d.\ data in prior work, but the analysis can be adjusted for the current setting to get the same bounds.) Thus our contribution is providing guarantees for the simple iterative trimming method but does not provide better than existing results.
\end{remark}

\section{LINEAR REGRESSION}
\label{LR}


Given observations $\left\{\left(\boldsymbol{x}_{i}, y_{i}\right)\right\}_{i=1}^{n}$ from the linear model
\begin{equation}
\label{model}
y_{i}=\boldsymbol{x}_{i}^{T} \boldsymbol{\beta}^{\star}+\epsilon_{i}, \quad \forall 1 \leq i \leq n,
\end{equation}
our goal is to estimate $\boldsymbol{\beta}^{\star}$.
Here $\left\{\epsilon_{i}\right\}_{i=1}^{n}$ are independently distributed with expectation $0$ and variances $\{\sigma_i^2\}$, and $\left\{\boldsymbol{x}_{i}\right\}_{i=1}^{n}$ are independent with $\left\{\epsilon_{i}\right\}_{i=1}^{n}$ and satisfy some regularity conditions described below.
Denote the stack matrix $(\boldsymbol{x}_{1}^{T}, \cdots,\boldsymbol{x}_{n}^{T})^{T}$ as $X$, the noise vector $\left(\epsilon_{1}, \cdots, \epsilon_{n}\right)^{T}$ as $\boldsymbol{\epsilon}$ and the response vector $\left(y_{1}, \cdots, y_{n}\right)^{T}$ as $\boldsymbol{y}$. 

\begin{assumption}
\label{a1}
Assume $\|\boldsymbol{x}_i\|_2 = 1$ for all $i$. 
Define
$$
\psi^{-}(k)=\min_{S: |S|=k} \lambda_{\min }\left(X_S^{T}  X_S\right),
$$
where $X_S$ is the submatrix of $X$ consisting of rows indexed by $S \subseteq [n]$.
Assume that for $k= \Omega(n), \psi^{-}(k)\ge k/c_1$ for a constant $c_1 > 0$. 
\end{assumption}

\begin{remark}
$\|\boldsymbol{x}_i\|_2 = 1$ is assumed without loss of generality, as we can always normalize $(\boldsymbol{x}_i, y_i)$, without affecting the assumption on $\epsilon_i$'s. 
The assumption on $\psi^{-}(k)$ states that every large enough subset of $X$ is well conditioned, and has been used in previous work, e.g.,~\cite{bhatia2015robust,shen2019learning}. It is worth mentioning that the uniformity over $S$ assumed is not for convenience but for necessity since the covariances of samples are unknown. Still, the assumption holds under several common settings in practice. For example, by Theorem 17 in~\cite{bhatia2015robust}, this regularity is guaranteed for $c_1$ close to $1$ w.h.p.\ when the rows of $X$ are i.i.d.\ spherical Gaussian vectors and $n$ is sufficiently large.
\end{remark}

\begin{remark}
In our setting, the noise terms are  independent  but  not  necessarily identically distributed.
It is also referred to as heteroscedasticity in linear regression by~\cite{rao1970estimation} and~\cite{horn1975estimating}.
\end{remark}

\subsection{Iterative Trimmed Squares Minimization}
We now apply iterative trimming to linear regression.
The first step is to use the square loss, i.e. let $f_{\boldsymbol{\beta}}\left(\boldsymbol{x}_{i}, y_i \right) =\left( y_i-\boldsymbol{x}_{i}^{T} \boldsymbol{\beta} \right)^2$, then the form of the trimmed loss estimator turns to
\begin{equation}
\label{TLE}
\hat{\boldsymbol{\beta}}^{(\mathrm{TL})}=\argmin _{\boldsymbol{\beta}, S :|S|=\lceil\alpha n\rceil} \sum_{i \in S} \left( y_i-\boldsymbol{x}_{i}^{T} \boldsymbol{\beta} \right)^2.
\end{equation}
Such $\hat{\boldsymbol{\beta}}^{(\mathrm{TL})}$ is first introduced by~\cite{rousseeuw1984least} as least trimmed squares estimator and its statistical efficiency has been studied in previous literature. However, the principal shortcoming is also its high computational complexity; see, e.g.,~\cite{mount2014least}. Hence, we again use iterative trimming.  Let 
$$
\hat{\boldsymbol{\beta}}_S^{(\mathrm{LS})}=\arg \min _{\boldsymbol{\beta}} \sum_{i \in S} \left(y_i-\boldsymbol{x}_i^{\top} \boldsymbol{\beta}\right)^{2}
$$ 
which is the least square estimator obtained by the sample set $S$. When $S=[n]$, we omit the subscript $S$ and write it as $\hat{\boldsymbol{\beta}}^{(\mathrm{LS})}$. The resulting algorithm is called Iterative Trimmed Squares Minimization and described in Algorithm \ref{ITSM}. 

\begin{algorithm}[t] 
	\caption{Iterative Trimmed Squares Minimization (ITSM)} 
	\label{ITSM} 
	\begin{algorithmic}[1]
		\REQUIRE{Samples $\left\{\left(\boldsymbol{x}_{i}, y_{i}\right)\right\}_{i=1}^{n},$ number of rounds $T,$ fraction of samples $\alpha$}
		
		\STATE{$\boldsymbol{\beta}_{0} \leftarrow \hat{\boldsymbol{\beta}}^{(\mathrm{LS})}~$}
		\FOR{$t=0, \cdots, T-1$}
		\STATE Choose samples with smallest current loss: 
		$$S_{t} \leftarrow \argmin _{S :|S|=\lceil\alpha n\rceil} \sum_{i \in S}\left(y_i-\boldsymbol{x}_i^{\top} \boldsymbol{\beta}_{t}\right)^{2}$$\
		\STATE $\boldsymbol{\beta}_{t+1}=\hat{\boldsymbol{\beta}}_{S_t}^{(\mathrm{LS})}$
		\ENDFOR
		\ENSURE{ $\boldsymbol{\beta}_T$}
	\end{algorithmic}
\end{algorithm}

\subsection{Theoretical Guarantees}

The key idea for the analysis is similar to that for mean estimation: the selected set $S$ has sufficiently large overlap with the set $S^{\star}$ of $\lceil \alpha n \rceil$ ``good" points with smallest noises, while the points in $S \backslash S^{\star}$ are not that ``bad" by the selection criterion. Therefore, the algorithm makes progress in each iteration.

\begin{lemma} 
\label{l-general}
Under Assumption \ref{a1}, given ITSM $\alpha \geq \frac{4 c_1}{1 + 4 c_1}$, with probability at least $1-\frac{1 + 4 c_1}{4 c_1 n}$, we have 
\begin{equation}
\label{l-bound}
\left\|\boldsymbol\beta_{t+1}- \boldsymbol\beta^{\star}\right\|_{2} \leq \frac{1}{2}\left\|\boldsymbol\beta_t-\boldsymbol\beta^{\star}\right\|_{2}+ 2 c_1 \sigma_{(\lceil \alpha n \rceil)}.
\end{equation}
\end{lemma} 

\begin{proof}
First we introduce some notations. Define
$$
\psi^{+}(k)=\max_{S: |S|=k} \lambda_{\max }\left(X_S^T X_S \right)
$$
where $X_S$ is the submatrix of $X$ consisting of rows indexed by $S$.
Note that $\psi^{+}(k) \le k$, since for any $S$ of size $k$, by  $\left\|\boldsymbol{x}_i\right\|_2 = 1$, $\lambda_{\max}\left(X_S^{T} X_S\right) $ is bounded by
\begin{align*}
\textrm{Tr}\left(X_S^{T} X_S\right) = \textrm{Tr}\left(X_S X_S^{T}\right) 
    & = \sum_{i \in S} \left\|\boldsymbol{x}_i\right\|^2_2 = k.
\end{align*}
Denote $W_t$ as the diagonal matrix where $W_{t,ii} = 1$ if the $i$-th sample is in set $S_t$, otherwise $W_{t,ii} = 0$. Let $S^{\star}$ be a subset of $\left\{i:\sigma_i \leq \sigma_{(\lceil \alpha n \rceil)}\right\}$ with size $\lceil \alpha n \rceil$ and denote $W^{\star}$ as the diagonal matrix w.r.t. $S^{\star}$.

Under Assumption~\ref{a1}, $X_{S_t}^T X_{S_t} =  X^{T}W_{t}X$ is nonsingular, so we have
$
\boldsymbol{\beta}_{t+1}=\left(X^{T}W_{t}X\right)^{-1}X^{T}W_{t}\boldsymbol{y},
$
where we have used  $W_{t}^{2}=W_{t}$. Then
\begin{align*} 
\boldsymbol{\beta}_{t+1}&=\boldsymbol{\beta}^{\star}+\left(X^{T}W_{t}X\right)^{-1}X^{T}W_{t}\boldsymbol{\epsilon}\\
&=\boldsymbol{\beta}^{\star}+\left(X^{T}W_{t}X\right)^{-1}X^{T}W_{t}\left((I-W^{\star})\boldsymbol{\epsilon}+W^{\star}\boldsymbol{\epsilon}\right).
\end{align*} 
Therefore, the error can be bounded by:
\begin{align*} 
& \left\|\boldsymbol{\beta}_{t+1}-\boldsymbol{\beta}^{\star}\right\|_{2}
\\
=& \left\| \left(X^{T}W_{t}X\right)^{-1}X^{T}\left(W_{t}(I-W^{\star})\boldsymbol{\epsilon}+W_{t}W^{\star}\boldsymbol{\epsilon}\right)\right\|_{2}\\
\leq & \frac{1}{\psi^{-}(\lceil \alpha n \rceil)} \left(\underbrace{\left\| X^{T}W_{t}(I-W^{\star})\boldsymbol{\epsilon}\right\|_{2}}_{\mathcal{T}_{1}}+\underbrace{\left\|X^{T}W_{t}W^{\star}\boldsymbol{\epsilon}\right\|_{2}}_{\mathcal{T}_{2}}\right),
\end{align*} 
where we use the spectral norm inequality and triangle inequality
and the fact that $\operatorname{Tr}(W_t)=\lceil \alpha n \rceil$. In the following, we bound the two terms $\mathcal{T}_{1}$  and $\mathcal{T}_{2}$.

First, $\mathcal{T}_{1}$ can be bounded as:
\begin{align*} 
\mathcal{T}_{1}&=\left\| X^{T}W_{t}(I-W^{\star})\boldsymbol{\epsilon}\right\|_{2}\\
&\leq \left\|X^{T}W_{t}(I-W^{\star})X(\boldsymbol{\beta}_t-\boldsymbol{\beta}^{\star}))\right\|_{2}\\
&~~~+\left\|X^{T}W_{t}(I-W^{\star})(\boldsymbol{y}-X\boldsymbol{\beta}_t)\right\|_{2}\\
&\leq \psi^{+}(\left|S_{t} \backslash S^{\star}\right|)\left\|\boldsymbol{\beta}_t-\boldsymbol{\beta}^{\star}\right\|_{2}\\
&~~~+\left\|X^{T}W_{t}(I-W^{\star})(\boldsymbol{y}-X\boldsymbol{\beta}_t)\right\|_{2},
\end{align*} 
since $\operatorname{Tr}((I-W^{\star})W_{t})=\left|S_{t} \backslash S^{\star}\right|$.

By the fact that $|S_t \backslash S^{\star}| = |S^{\star} \backslash S_t|$, there exists a bijection between $S_t \backslash S^{\star}$ and $S^{\star} \backslash S_t$. Denote the image of $i \in S_t \backslash S^{\star}$ as $k_i$. Since the loss $(y_i-\boldsymbol{x}_i^{T}\boldsymbol{\beta}_t)^2$ of sample in $S_t \backslash S^{\star}$ is less than that of sample in $S^{\star} \backslash S_t$, we have $\left|y_i-\boldsymbol{x}_i^{T}\boldsymbol{\beta}_t\right| \leq \left|y_{k_i}-\boldsymbol{x}_{k_i}^{T}\boldsymbol{\beta}_t\right|$ for any $i \in S_t \backslash S^{\star}$. Hence we can write
\begin{align*}
~~~&\left\|X^{T}W_{t}(I-W^{\star})(\boldsymbol{y}-X\boldsymbol{\beta}_t)\right\|_{2} \\
= & \left\|\sum_{i \in S_t \backslash S^{\star}}(y_i-\boldsymbol{x}_i^{T}\boldsymbol{\beta}_t) \boldsymbol{x}_i \right\|_{2}\\
\leq & \left\|\sum_{i \in S_t \backslash S^{\star}}n_i (y_{k_i}-\boldsymbol{x}_{k_i}^{T}\boldsymbol{\beta}_t) \boldsymbol{x}_i \right\|_{2}\\
\leq &\left\|\sum_{i \in S_t \backslash S^{\star}}n_i \epsilon_{k_i} \boldsymbol{x}_i \right\|_{2} + \left\|\sum_{i \in S_t \backslash S^{\star}}n_i \boldsymbol{x}_i {\boldsymbol{x}_{k_i}}^{T}(\boldsymbol{\beta}_t-\boldsymbol{\beta}^{\star}) \right\|_{2}
\end{align*}
where $n_i$ is either $1$ or $-1$. 

By the assumption that $\left\|\boldsymbol{x}_i \right\|_{2} = 1$, 
\begin{align*}
\left\|\sum_{i \in S_t \backslash S^{\star}}n_i \epsilon_{k_i} \boldsymbol{x}_i \right\|_{2} 
\leq & \sum_{i \in S_t \backslash S^{\star}} |\epsilon_{k_i}| =  \sum_{i \in S^{\star} \backslash S_t} |\epsilon_i|.
\end{align*}
Meanwhile,  
\begin{align*}
 \left\|\sum_{i \in S_t \backslash S^{\star}}n_i \boldsymbol{x}_i {\boldsymbol{x}_{k_i}}^{T}(\boldsymbol{\beta}_t-\boldsymbol{\beta}^{\star}) \right\|_{2} \le s_{\max} \cdot \left\| \boldsymbol{\beta}_t-\boldsymbol{\beta}^{\star} \right\|_{2}.
\end{align*}
Here $s_{\max} := s_{\max}\left(\sum_{i \in S_t \backslash S^{\star}}n_i \boldsymbol{x}_i {\boldsymbol{x}_{k_i}}^{T}\right)$ denotes the maximum singular value and is bounded as follows.

\begin{claim}
$s_{\max} \leq \psi^{+}\left(|S_t \backslash S^{\star}|\right)$.
\end{claim}
\begin{proof}
The proof is implicit in the proof for Theorem $7$ in~\cite{shen2019learning}.
With matrix notations, $\sum_{i \in S_t \backslash S^{\star}}n_i \boldsymbol{x}_i {\boldsymbol{x}_{k_i}}^{T}$ can be then written as $X^{T}W_t(I-W^{\star})N P X$, where $N$ is  diagonal with diagonal entries in $\{1, -1\}$ and $P$ is some permutation matrix. Then
\begin{align*}
s_{\max} 
= \max_{\|u\|_2 = 1, \|v\|=1} u^T X^{T}W_t(I-W^{\star})N P X v.
\end{align*}
Let $\tilde{u} =Xu$ and $\tilde{v}=Xv$, $s_{\max}$ is bounded by
\begin{align*}
\sum_{i \in S_t \backslash S^{\star}} |\tilde{u}_{r_i} \tilde{v}_{t_i}| \le 
\max\{\sum_{i \in S_t \backslash S^{\star}} \tilde{u}_{r_i}^2, \sum_{i \in S_t \backslash S^{\star}} \tilde{v}_{t_i}^2\}
\end{align*}
for some sequences $\{r_i\}$ and $\{t_i\}$. So $s_{\max}$ is bounded by the larger of the maximum singular values of $X^{T}W_t(I-W^{\star})X $ and $X^{T}P^T N^T W_t(I-W^{\star})N P X $, which is bounded by $\psi^{+}\left(|S_t \backslash S^{\star}|\right)$. 
$\blacksquare$
\end{proof}
With this claim, we can derive
$$
\mathcal{T}_{1} \leq \sum_{i \in S^{\star} \backslash S_t} |\epsilon_i| + 2\psi^{+}\left(|S_t \backslash S^{\star}|\right) \left\|\boldsymbol{\beta}^{\star}-\boldsymbol{\beta}_t\right\|_2.
$$

Next, $\mathcal{T}_{2}$ can be bounded as:
\begin{align*}
\mathcal{T}_2 = \left\|\sum_{i \in S_{t} \cap S^{\star}}\epsilon_i \boldsymbol{x}_i \right\|_2 \leq \sum_{i \in S_{t} \cap S^{\star}} \left|\epsilon_i\right|.
\end{align*}

Combining the inequalities above, we have 
\begin{align*} &~~~\left\|\boldsymbol{\beta}^{\star}-\boldsymbol{\beta}_t\right\|_2 = \frac{1}{\psi^{-}(\lceil \alpha n \rceil)} \left(\mathcal{T}_1 + \mathcal{T}_2\right)\\
&\leq \frac{1}{\psi^{-}(\lceil \alpha n \rceil)}
\left(\sum_{i \in S^{\star}} \left|\epsilon_i\right| + 2\psi^{+}\left(|S_t \backslash S^{\star}|\right) \left\|\boldsymbol{\beta}^{\star}-\boldsymbol{\beta}_t\right\|_2\right)\\
&=\underbrace{\frac{2\psi^{+}\left(|S_t \backslash S^{\star}|\right)}{\psi^{-}(\lceil \alpha n \rceil)}}_{\kappa}
\left\|\boldsymbol{\beta}^{\star}-\boldsymbol{\beta}_t\right\|_2 + \frac{\sum_{i \in S^{\star}} \left|\epsilon_i\right|}{\psi^{-}(\lceil \alpha n \rceil)}.
\end{align*}
By assumption, there exist a constant $c_1$ such that $\frac{\alpha n}{\psi^{-}(\lceil \alpha n \rceil)} \leq c_1$ and $\frac{\psi^{+}(\left|S_{t} \backslash S^{\star}\right|)}{\psi^{-}(\lceil \alpha n \rceil)} \leq c_1\cdot \frac{\left|S_{t} \backslash S^{\star}\right|}{\lceil \alpha n \rceil}$. Without loss of generality, assume $\alpha n$ is an integer. Since $\left|S_{t} \backslash S^{\star}\right| \leq (1-\alpha)n$, when $\alpha \geq \frac{4 c_1}{1 + 4 c_1}$,  
$
\kappa \leq 2 c_1\cdot \frac{1-\alpha}{\alpha} \leq \frac{1}{2}.
$
And $\sum_{i \in S^{\star}} \left|\epsilon_i\right|$ is bounded by $2|S^{\star}|\sigma_{(\lceil \alpha n \rceil)}$ with probability at least $1-\frac{1}{|S^{\star}|}$, using Markov's inequality.
$\blacksquare$
\end{proof}

\begin{theorem}
\label{l-final}
Under Assumption \ref{a1}, given ITSM with $\alpha \geq \frac{4 c_1}{1 + 4 c_1}$ and $T=\Theta\left(\log_{2} \frac{\left\|\boldsymbol\beta_0-\boldsymbol\beta^{\star}\right\|_2}{\sigma_{(\lceil \alpha n \rceil)}}\right)$, it holds that
\begin{equation}
\label{l-error}
\left\|\boldsymbol\beta_T-\boldsymbol\beta^{\star}\right\|_2 \leq c c_1 \sigma_{(\lceil \alpha n \rceil)}
\end{equation}
with probability at least $1-T\frac{1 + 4 c_1}{4 c_1 n}$. 
\end{theorem}

\begin{proof}
It follows from Lemma~\ref{l-general} by an argument similar to that for Theorem~\ref{m-final}.
$\blacksquare$
\end{proof}

\begin{remark}
When $\boldsymbol{x_i}$'s are i.i.d.\ spherical Gaussians, Assumption~\ref{a1} can be satisfied with $c_1$ close to 1. Then we require $\alpha \ge 4/5$ and the error bound holds with probability $\ge 1- 5T/(4n)$, similar to that for mean estimation. 
Also, (\ref{l-error}) is obtained without extra assumption on noise $\epsilon_i$ except for assuming its second order moment exists. 
If $\epsilon_i \sim \mathcal{N}(0,\sigma_i^2)$, the previous bound holds with a higher probability $1-e^{-c^{\prime}n}$.
\end{remark}

\begin{remark}
We now discuss existing results.
~\cite{pensia2019estimating} also adapted its mean estimation methodology to linear regression. When $\boldsymbol{x}_i$'s are from a multivariate Gaussian with covariance matrix $\Sigma$, w.h.p.\ it obtained error bound 
$\frac{c^{\prime}n \sigma_{(cd\log n)}}{\sqrt{\lambda_{\min}(\Sigma)}}$. Again, the result depends on $\sigma_{(cd\log n)}$ with an additional factor $n$, while ours depends on $\sigma_{(\alpha\log n)}$ for a constant $\alpha$ (our bound will also have a $\frac{1}{\sqrt{\lambda_{\min}(\Sigma)}}$ factor for such Gaussians). However, their run time is $\mathcal{O}(n^d)$, exponential in the dimension $d$, while ours is polynomial.

There also exist studies for robust linear regression in the adversary setting, where a small fraction of points are being corrupted by an adversary (e.g., \cite{bhatia2015robust,liu2018high,shen2019learning,diakonikolas2019efficient}). 
It is unclear if their results directly apply to our setting, since they have additional assumptions on the data. 
\end{remark}


\section{Experiments}

\subsection{Mean Estimation}
To validate Theorem $\ref{m-final}$, we repeat the procedure that first generating samples $\left\{\boldsymbol{x}_{i}\right\}_{i=1}^n$ under designed distribution $\left\{F_{i}\right\}_{i=1}^n$, then run Algorithm $\ref{ITM}$ with fraction $\alpha=\frac{4}{5}$ and iteration $T=20$ and finally output error $\left|\boldsymbol\mu_T-\boldsymbol\mu^{\star}\right|$. We report the average error over $R$ repetitions; we set $R=200$ for the univariate case and $R=20$ for the multivariate case. For comparison, we also report the error of the empirical mean over the first $\alpha n$ samples with smallest variance (or norm of covariance matrix), which is the estimator given an oracle knowing all the covariances, and thus referred to as Oracle Mean (OM). 
It is easy to be seen that the latter average is only relevant to $\left\{\sigma_{(i)}\right\}_{i=1}^{\alpha n}$ (or $\left\{\lambda_{(i)}\right\}_{i=1}^{\alpha n}$) and regardless of the rest $(1-\alpha)n$ samples. 

\noindent\textbf{Univariate Case}
We first present experiments when $d=1$ under two designed settings of the entangled distributions $\left\{F_{i}\right\}_{i=1}^n$. 

\noindent{\bf Setting 1} For $i \leq \alpha n$, $F_{i}=\mathcal{N}(0,1)$ and $F_{i}=\mathcal{N}(0,i^2)$ for $i > \alpha n$.

\noindent{\bf Setting 2} For $i \leq \alpha n$, $F_{i}=\mathcal{N}(0,(\log i)^2)$ and $F_{i}=\mathcal{N}(0,i^2)$ for $i > \alpha n$.

Figure $\ref{f1}$ shows the performances of ITM and OM.
In both settings, ITM obtains roughly the same average error as OM, showing that the error bound of ITM is regardless of all $\sigma_i \geq \sigma_{(\lceil \alpha n \rceil)}$. Besides, ITM actually converges to the truth $\boldsymbol{\mu}^{\star}$, in the same rate as OM. It suggests ITM can achieve a vanishing error, at least in some special cases. This is left as a future direction. 

\begin{figure}[t]
	\centering
	\includegraphics[scale=0.35]{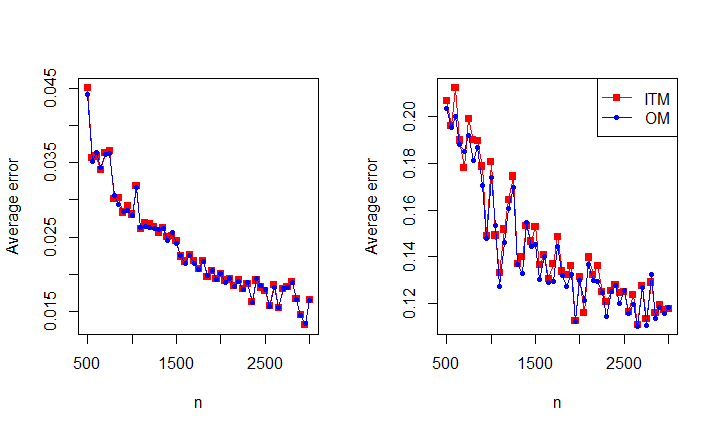}
	\caption[Figure]{Univariate mean estimation. Left: setting 1; Right: setting 2. $x$-axis: sample size; $y$-axis: error after $T=20$ iterations. 
	}
	\label{f1} 
\end{figure}

\noindent\textbf{Multivariate Case}
For the multivariate case, we set $d=10$ and provide two simulation settings of $\left\{F_{i}\right\}_{i=1}^n$.

\noindent{\bf Setting 3 (radically symmetric)}  $F_{i}=\mathcal{N}(0,I_{10})$ for $i \leq \alpha n$, and $F_{i}=\mathcal{N}(0,100 \cdot I_{10})$ for $i > \alpha n$.

\noindent{\bf Setting 4 (radically asymmetric)} Let $\Sigma_0$ be a stochastic positive definite matrix generated as follows.
(1) Set the diagonal entries to 1.
(2) Off-diagonal entries are set to zero or nonzero with equal probability. For each nonzero off-diagonal entry, it is sampled from the uniform distribution on interval $(-0.5,0.5)$. Then we make it symmetric by forcing the lower triangular matrix equal to the upper triangular, and then add a diagonal matrix ${c}\textbf{I}_{10}$ to make sure it is positive definite, where ${c}$ is chosen to make the smallest eigenvalue of the matrix equal to 0.2. Finally, let $F_{i}=\mathcal{N}(0,\Sigma_0)$ for $i \le \alpha n$ and $F_{i}=\mathcal{N}(0,100 \cdot \Sigma_0)$ for $i > \alpha n$.   

Figure~\ref{f2} shows the results. 
Again, in both settings, ITM also obtain the same average error with OM. The results for setting 4 suggest that the radical symmetry assumption may not be needed. 

\begin{figure}[t]
	\centering
	\includegraphics[scale=0.35]{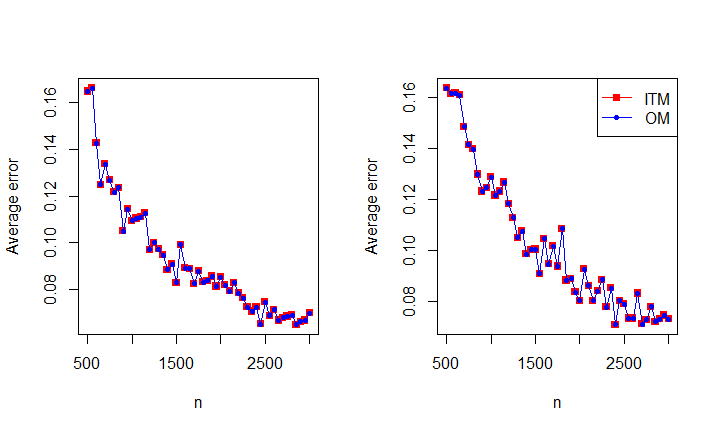} 
	\caption[Figure]{Multivariate mean estimation. Left: setting 3; Right: setting 4. $x$-axis: sample size; $y$-axis: error after $T=20$ iterations. 
	}   
    \label{f2}  
\end{figure}

\subsection{Linear Regression}
We now consider the linear regression problem with $d=100$. We generate data according to the model $(\ref{model})$, where each row of matrix $X$ is independently sampled from $\mathcal{N}(\boldsymbol{0},I_d)$ and we choose $\boldsymbol{\beta}^{\star}$ to be a random vector with $l_2$ norm 1. The noise vector is generated s.t.\ for $i \leq \alpha n$, $\epsilon_{i} \sim \mathcal{N}(0,1)$ and $\epsilon_{i} \sim \mathcal{N}(0,100)$ for $i > \alpha n$. We repeatedly generate data and run ITSM for $R=20$ times, then report the average errors of both ITSM and the least square estimator over the first $\alpha n$ samples with smallest noise variances, which we refer to as Oracle Least Square (OLS). We also set $\alpha = \frac{4}{5}$ and $T=20$ for simplicity though the smallest possible value of $\alpha$ in Lemma~$\ref{l-general}$ is dependent on $X$.

\begin{figure}[h]
	\centering
	\includegraphics[scale=0.35]{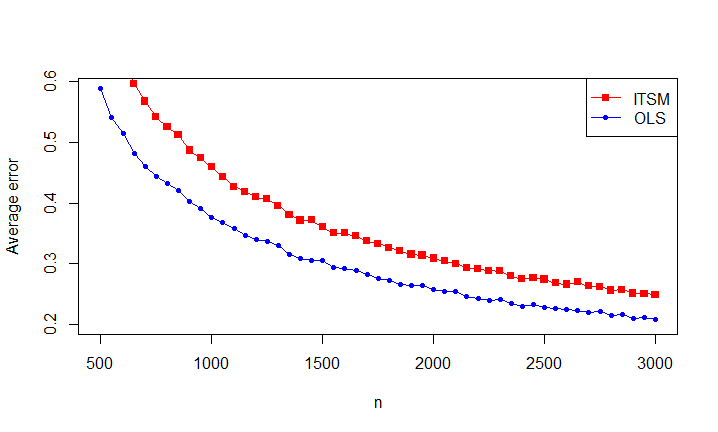}
	\caption[Figure]{Linear Regression. $x$-axis: sample size; $y$-axis: error after $T=20$ iterations. 
	}     
	\label{f3}
\end{figure}

Figure~\ref{f3} shows the results. Similar to mean estimation, the iterative trimming method has error closely tracking those of the oracle method. This again verifies the effectiveness of iterative trimming and provide positive support for our analysis.

\section*{Acknowledgement}
This work was supported in part by FA9550-18-1-0166. The authors would also like to acknowledge the support provided by the University of Wisconsin-Madison Office of the Vice Chancellor for Research and Graduate Education with funding from the Wisconsin Alumni Research Foundation.

\bibliographystyle{plainnat}
\bibliography{ref}

\end{document}